\documentclass[a4paper,12pt]{article}

\input xy
\xyoption{all}
\usepackage{amsbsy,eucal, amssymb,amsmath}
\usepackage{latexsym,amsfonts,amsthm,dsfont}
\usepackage{authblk}

\newcommand{\set}[1]{\{#1\}}
\newcommand{\Set}{{\sf Set}}
\newcommand{\Meas}{{\sf Meas}}
\newcommand{\C}{{\cal C}}

\newcommand{\N}{{\cal N}}
\newcommand{\RR}{\mathds{R}}

\newcommand{\ba}{{\bf a}}

\newcommand{\bsigma}{{\boldsymbol \sigma}}
\newcommand{\bb}{{\bf b}}

\theoremstyle{definition}
\newtheorem{definition}{\bf Definition}[section]
\newtheorem{examplex}[definition]{\bf Example}
\newenvironment{example}
{\pushQED{\qed}\examplex}
{\popQED\endexamplex}

\newtheorem{proposition}[definition]{Proposition} 
\newtheorem{theorem}[definition]{Theorem} 
\newtheorem{lemma}[definition]{Lemma}

\begin{document}
\title{Superrational types}

\author[1,2]{Fernando A.~Tohm\'e }
\author[2,3]{Ignacio D. Viglizzo}
\affil[1]{Departmento de Econom\'ia, Universidad Nacional del Sur (UNS), Bah\'ia Blanca, Argentina}
\affil[2]{Instituto de Matem\'atica de Bah\'ia Blanca (INMABB), Universidad Nacional del Sur-CONICET, Bah\'ia Blanca,	Argentina}
\affil[3]{Departmento de Matem\'atica, Universidad Nacional del Sur (UNS), Bah\'ia Blanca, Argentina}

\maketitle

\begin{abstract}
We present a formal analysis of Douglas Hofstadter's concept of \emph{superrationality}.  We start by defining superrationally justifiable actions, and study them in symmetric games. We then model the beliefs of the players, in a way that leads them to different choices than the usual assumption of rationality by restricting the range of conceivable choices. These beliefs are captured in the formal notion of \emph{type} drawn from epistemic game theory.  The theory of coalgebras is used to frame type spaces and to account for the existence of some of them.  We find conditions that guarantee superrational outcomes.\end{abstract}

\section{Introduction}

Game Theory has, traditionally assumed that agents make decisions according to the criterion of full rationality \cite{Gintis2009}. In contrast to the case of {\em bounded rationality} (\cite{Simon72}, \cite{Rubinstein98}), agents choose the best actions to achieve their goals without limitations to their ability to do so. The research on behavioral factors in decision-making, initiated by Kahneman and Tversky (\cite{KT}) lead to a dearth of results under weaker rationality assumptions. On the other hand, the interest in overcoming ``undesirable'' results such as inefficient equilibria in social dilemmas lead to alternative concepts, diverging from the rationality assumption \cite{Howard71}, \cite{Rapoport67}. One of the most intriguing of these notions is Douglas Hofstadter's {\em superrationality} \cite{hofstadter83dilemmas}:

\begin{quote}
\textit{  Any number of ideal rational thinkers faced with the same situation and undergoing the same throes of reasoning agony will necessarily come up with the identical answer eventually, so long as reasoning alone is the ultimate justification for their conclusion.}
\end{quote}

In a game context, the conditions for ``being in the same situation'' translates in the implicit requirements that the players can choose from the same pool of actions, and  furthermore, that these actions have the same consequences for all the players (otherwise, no common election may be conceived). Then, in such a game, the definition of superrationality leads in the first place to the choice of a profile of identical strategies. Furthermore, among all these profiles, since the agents intend to maximize their payoffs, they will choose the one that yields the highest payoff, the same for all of them. See \cite{hofstadter85metamagical} for Hofstadter's extensive discussion on the social and political implications of superrationality as well as on how it works in different contexts.

To model superrational behavior without resorting to arbitrary constrains and to justify internal reasoning processes that can lead the  agents to take it, we need to include  their beliefs. This puts the problem in the frame of {\em epistemic game theory}. In this field the concept of {\em type} of an agent plays a fundamental role, capturing the structure of beliefs and knowledge that, jointly with their preferences, lead to the agent's decision.

The goal of this paper is thus to explore formalizations of superrationality, taking into account explicitly the types of the players. Since beliefs can be modeled in different ways, we adopt alternative characterizations of the types of the agents, and show that they can support superrationality. 

The plan of the paper is as follows. Section 2 introduces the notions of superrationally justifiable actions and superrational profiles and compares the latter to Nash equilibria. Superrational profiles are defined based purely on the payoff matrix of a game.  For the players to converge on a superrational profile, a unique superrationally justifiable action must be present in the game. In symmetric games the existence of superrationally justifiable actions can be asserted. The example of the Prisoner's dilemma shows that sometimes the concept of superrationality can lead to better outcomes for the players. All of the results can be extended to mixed strategies. In Section 3 we start analyzing epistemic conditions that could lead the players to aim for the superrational outcomes, both in pure and mixed strategies, modeling the types of the players.  For this we use measure theory to represent the beliefs of the players, and draw from the theory of coalgebra to establish the existence of some of the required type spaces. We compare the epistemic conditions with those posed by Aumann and Brandenburger for reaching Nash equilibria. In Section 4 we move on to strategic belief models, where the beliefs of the players are represented by sets of possible states of the world. Section 5 is dedicated to include the case in which each player draws her type from a different type space. An equivalence relation still allows to give an interpretation of superrationality in this context. Finally, we present some conclusions.

\section{Superrational profiles}

To get to a definition of superrationality based purely on the payoff matrix of a game, let us introduce some preliminary definitions:

\begin{definition}\label{definitionGame}
Let $G = \langle I, \{A_i\}_{i \in I}, \{\pi_i\}_{i \in I} \rangle$ be a {\em game}, where $I=\set{1,\ldots, n}$ is a set of {\em players} and $A_i, i\in I$ is a finite set of {\em actions} for each player. An {\em action profile}, $\ba=(a_1, \ldots, a_n)$ is an element of $A = \prod_{i \in I} A_i$.
In turn, $\pi_i : A \rightarrow \mathds{R} $ is player $i$'s payoff.
\end{definition}

To even consider superrationality in a game, all the players must have the same set of actions available, so we restrict ourselves to games in which this is the case. An action of a player is \textit{superrationally justifiable} if under the assumption that all the players will coincide in their choices, the payoff is maximized.

 \begin{definition}
 	In a game $G$ in which for every $i,j \in I$, $A_i = A_j$, an action $a^{*}  \in A_i$ is {\em superrationally justifiable},  iff  for every $i \in I$ and every $a \in A_i	$, $\pi_i(a^{*},\ldots,a^*) \geq \pi_i(a,\ldots,a)$.
 \end{definition}
 
 It follows from the definition that the superrationally justifiable actions are the same for all the players. Assuming that all players choose the same superrationally justifiable action, we obtain a profile in the `diagonal' of $G$, denoted by $\Gamma_G \subseteq A$, of the profiles of the form $(a,\ldots, a)$. A new solution concept can be thus defined:
 \begin{definition}
 	A profile $\ba^{*} \in \Gamma_G$ is {\em superrational}, indicated $\ba^{*} \in \mathcal{SR}_{G} \subseteq A$ iff $\pi_i(\ba^{*}) \geq \pi_i(\ba)$ for every $i \in I$ and every $\ba \in \Gamma_G$.
 \end{definition}

Since superrational outcomes are compared to Nash equilibria, let us recall the definition of this solution concept:

\begin{definition}\label{Nash}
	An action profile  $\ba^{*}= (a^{*}_1, \ldots, a^{*}_n)$ in a game $G$ is a Nash equilibrium, $\ba^*\in{\mathcal{NE}_G}$, if and only if for every $i$ and every alternative $a_i \in A_i$
	
	\[\pi_i(a^*_1,\ldots,a^*_i, \ldots a^*_n) \geq \pi_{i}(a^*_1,\ldots,a_i, \ldots a^*_n)\]
\end{definition}

Notice that, according to Nash's theorem \cite{nash51noncooperative}, a Nash equilibrium always exist in $\Delta A_1 \times \ldots \times \Delta A_n$, where $\Delta A_i$ is the class of probability distributions over $A_i$ ({\em mixed strategies}). The payoffs in the game in which these strategies are used obtain by finding the expected values of the payoffs defined on $A$. That is, if $\bsigma = (\sigma_1, \ldots, \sigma_n) \in$$\Delta A_1 \times \ldots \times \Delta A_n$, the expected payoff of $i$ is $E\pi_i(\bsigma) = \sum_{(a_1, \ldots, a_n) \in A}\left( \pi_i(a_1, \ldots, a_n) \prod_{k=1}^n \sigma_k(a_k)\right)$.

If a superrational profile $\ba\in \mathcal{SR}_G$ is reached, each of its actions are superrationally justified, but the converse is not true.

\begin{example}\label{justin}
Consider the following game:
 \begin{center}
 	\begin{tabular}{c|c|c|c}
 		&a&b&c\\\hline
 		a& $(2,3)$&$(0,0)$&$(0,0)$\\ \hline
 		b& $(0,0)$&$(2,3)$&$(0,0)$\\ \hline
 		c& $(0,0)$&$(0,0)$&$(2,2)$
 	\end{tabular}
 \end{center}
  Both $a$ and $b$ are superrationally justifiable actions and $\ba = (a,a)$ and $\bb =(b,b)$ are both superrational profiles and Nash equilibria. On the other hand, $(c,c)$ is a Nash equilibrium but not superrational.  Neither $(a,b)$ nor $(b,a)$ belong to $\mathcal{SR}_G$, even though their individual actions are superrationally justifiable.
 \end{example}

Superrationally justifiable actions are not always available, and even if they are, they would not be chosen by any agent interested in maximizing her payoffs:

\begin{example}\label{BoS} Consider the Battle of the Sexes (\cite{LuceRaiffa57}), in which each player can choose either \textit{Box} or \textit{Ballet}. It is apparent that none of the actions is superrationally justifiable.
\begin{center}
	\begin{tabular}{c|c|c}
		&\textit{Box}&\textit{Ballet}\\\hline
		\textit{Box}& $(2,1)$&$(0,0)$\\ \hline
		\textit{Ballet}& $(0,0)$&$(1,2)$
	\end{tabular}
\end{center}
Thus $\mathcal{SR}$ is empty for this game.
\end{example}
\begin{example}\label{nonsym} In contrast, in the following game,  $\mathcal{SR}=\set{(b,b)}$, but clearly the profile $(b,a)$ is better for both players:
	\begin{center}
		\begin{tabular}{c|c|c}
			&a&b\\\hline
			a& $(0,0)$&$(0,0)$\\ \hline
			b& $(2,2)$&$(1,1)$
		\end{tabular}
	\end{center}
\end{example}

In non-symmetric games, even if actions for different players have the same name, the preferences over the consequences they yield  may differ wildly from one player to another. Therefore, it makes sense to further restrict the class of games under consideration to \textit{symmetric} games. As a bonus, we get that for this class of games,  superrationally justifiable actions always exist.

\begin{definition}\label{symmetricGame} (\cite{tohme17symmetry})
	\label{permutationInvariant}
	A game $G$ in which for every $i,j \in I$, $A_i = A_j$  is {\em symmetric}  if the payoff function $\pi:A\to\RR^n$ is  invariant under permutations; this is, for any permutation $\tau$ (that is, a bijection $\tau :I\to I$)  and every action profile  $(a_1, \ldots, a_n)$ we have that for all $i\in I$,
	
	\begin{equation}\label{symmetric}
	\pi_i(a_1,\ldots, a_n)= \pi_{\tau^{-1}(i)}(a_{\tau(1)},\ldots, a_{\tau(n)}).\end{equation}
\end{definition}
Notice that the games in Examples \ref{justin}, \ref{BoS}, and \ref{nonsym} are not symmetric.

\begin{lemma}\label{same_payoff}
	In a symmetric game, if all players play the same action, they all get the same payoff.
\end{lemma}

\begin{proof}
	If all players play the same action, say $a_1$, then for any $i, j\in I$, let $\tau$ be a permutation such that $\tau^{-1}(i)=j$. Then \[\pi_i(a_1,\ldots,a_1)=\pi_j(a_1,\ldots,a_1).\]
\end{proof}

As a consequence of Lemma \ref{same_payoff}, since the set of actions is finite, we get that:

\begin{proposition}\label{existSRprofile}
	In any symmetric game there exist superrational profiles.
\end{proposition}
	
	\begin{proof}
		Since by Lemma \ref{same_payoff}, in all the profiles in $\Gamma_G$, all the players get an identical payoff, an action  that maximizes the payoff for one of them gives also the maximum for all of them.
	\end{proof}
	
The following examples of symmetric games show that the notions of Nash equilibria and superrational profiles are different, and it is not clear that one is better than the other:

\begin{example}\label{Bad}
	Consider the symmetric game given by the payoff matrix: 
	\begin{center}
		\begin{tabular}{c|c|c}
			&a&b\\\hline
			a& $(0,0)$&$(1,1)$\\ \hline
			b& $(1,1)$&$(0,0)$
		\end{tabular}
	\end{center}
	Here   $\mathcal{SR}=\set{(a,a),(b,b)}$, while ${\mathcal{NE}}=\set{(a,b),(b,a)}$ so the Nash equilibria are all outside the diagonal $\Gamma$.
\end{example}
\begin{example}\label{prisoner}
	 In the famous Prisoner's Dilemma, however, we have that $\mathcal{SR} = \{(C,C)\}$ and $\mathcal{NE}=\set{(D,D)}$ (where $C$ and $D$ are, respectively ``Cooperate'' and ``Defect'') so both are in the diagonal, but they are different.
 \begin{center}
  \begin{tabular}{c|c|c}
  &C&D\\\hline
  C& $(3,3)$&$(0,5)$\\ \hline
  D& $(5,0)$&$(1,1)$
  \end{tabular}
\end{center}

In this case, the superrational outcome is better for both players than the Nash equilibrium, while in the previous one, the opposite is true. Analogously, in the closely related and well known Traveler's Dilemma \cite{basu2007traveler}, the superrational outcome is much better for both players than the Nash equilibrium.
\end{example}

We now study the notion of superrational profiles in mixed strategies. Let $\Gamma^{\Delta}_{G}\subseteq   \Delta A_1 \times \ldots \times \Delta A_n$ be the set   \[\Gamma^{\Delta}_{G} = \{\sigma = (\sigma_1, \ldots, \sigma_n): \sigma_i = \sigma_j \ \mbox{for every} \ i,j \in I\}\] 

\begin{definition}\label{mixedNash}
Given a game $G$, a \textit{superrational profile of mixed strategies} is a profile  $\bsigma^{*} = (\sigma^{*}, \ldots, \sigma^{*}) \in \Gamma^{\Delta}_{G}$,  denoted $\bsigma^{*} \in \mathcal{SR}^\Delta_{G}$, if and only if  for every $\bsigma \in \Gamma^{\Delta}_{G}$ and $i\in I$, $E\pi_i(\bsigma^{*}) \geq E\pi_i(\bsigma)$. The mixed strategy $\sigma^*$ is then a \textit{superrationally justifiable mixed strategy}.
\end{definition}

\begin{example}
	The Chicken Game (\cite{RapoportChammah66}) involves another conflicting situation among two parties. Call the actions $S$ and $Y$ (for ``go straight'' and ``yield'', respectively) with payoffs:
	
	\begin{center}
		\begin{tabular}{c|c|c}
			&S&Y\\\hline
			S& $(-10,-10)$&$(1,-1)$\\ \hline
			Y& $(-1,1)$&$(0,0)$
		\end{tabular}
	\end{center}

	The superrational solution, in pure strategies is $\mathcal{SR} = \{(Y,Y)\}$, in which  both players think that the other will choose the same action as they do, and consequently, they are better off yielding.
	
Let us call $p \in [0,1]$ the probability of choosing $S$ while $Y$ has a probability $1-p$. Then $\Gamma^{\Delta} = \{(p,p): p \in [0,1]\}$. The superrational mixed strategy obtains by maximizing $E\pi_i(p,p)$ for each $i=1,2$. We have that:
	
	$$E\pi_1(p,p)\ =E\pi_2(p,p)= \ -10 p^2 + 1(p-p^2) + (-1)(p-p^2)  = \ -10 p^2.$$
	 This means that the expected payoffs are maximized at $p=0$, i.e. $\mathcal{SR}^\Delta = \{(0,0)\}$, which can be identified with the pure superrational strategy profile $(Y,Y)$.
\end{example}

The previous example might suggest that, as with Nash equilibria, pure strategy superrational outcomes support also superrationality in mixed strategy.
Such parsimonious behavior is not ensured, as shown in Hofstadter's {\em Platonia Dilemma} \cite{hofstadter83dilemmas}:

\begin{example}\label{platonia}
	
	Consider a situation in which $n$ individuals are asked either to send or not a letter to an umpire. Only if a single letter is
	sent, a prize of $\$1,\!000,\!000$ is awarded to the sender and $0$ to the other participants.
	If either no letter has been sent or more than one
	has been received by the umpire, each player gets nothing.
	
	This situation can be analyzed as a $n$-player game $G$ in which
	each $A_i = \{S, D\}$ where these actions are {\em send} ($S$) or {\em do not
		send} ($D$) a letter. The payoff is $1,\!000,\!000$ for the sender of the letter if a single player
	chooses $S$, and $0$ otherwise.
	
	The Nash equilibria in the game are the $\ba \in
	\prod_{i=1}^{n} A_i$ such that $ |\{i:a_i = S\}| \ge 1$. That is, each
	equilibrium obtains if and only if one or more  players submit a letter.
	The mixed strategies Nash equilibria are all those in $\Delta(\{S,D\})^n$ such that at least for one $i$, $\sigma_i(S)=1$, that is, at least one of the players chooses $S$ with certainty.

	Superrationality in pure strategies obtains in any profile in the
	diagonal, in particular when each agent chooses not to send a letter
	and obtains a $0$ payoff, or all of them send the letter, i.e. $\mathcal{SR}_G = \{(S, \ldots,S),(D, \ldots,D)\}$. A superrational agent who ponders mixed
	strategies instead, looks for all the profiles $(p_{1}^{*}, \ldots,
	p_{n}^{*})$ such that for each $i$, $p_{i}^{*} = \delta^{*}$ where
	$\delta^{*} \in [0,1]$ maximizes

	\[E\pi_i(\delta,\ldots,\delta) \ = \ 1,\!000,\!000\; \delta (1 - \delta)^{n-1}\]
	
	The first order condition yields
	\[(1-\delta)^{n-1} - (n-1)\delta (1- \delta)^{n-2} \ = \ 0\]
	
	\noindent assuming that $(1- \delta)^{n-2} \neq 0$ (i.e. $\delta
	\neq 1$) we get:
	
	\[ 1 - \delta - (n-1) \delta \ = \ 0\]
	
	That is, $\delta^{*} = \frac{1}{n}$. The expected payoff is then strictly positive, unlike the ones corresponding to the pure strategy superrational profiles. Thus, the latter cannot be seen as superrational mixed strategy solutions.
\end{example}

Since pure superrational profiles may not correspond to mixed superrational ones, we have to establish the existence of the latter ones independently:

\begin{proposition}\label{existence_mixed}
	In any symmetric game, there exist at least one superrational profile of mixed strategies.
\end{proposition}
\begin{proof}
	First of all, for every pair of players $i, j\in I$, and every mixed strategies profile $\bsigma \in \Gamma^{\Delta}_{G}$, say $\bsigma = (\sigma, \ldots, \sigma)$, $E\pi_i(\bsigma) = E\pi_j(\bsigma)$. To see this, take any permutation $\tau$  such that $\tau(i)=j$. Then the products $\prod_{i=1}^n \sigma(a_i)$ and $\prod_{i=1}^n \sigma(a_{\tau(i)})$ are equal, being just rearrangements one of the other. Then, by the symmetry of the game we have:  
	\[E\pi_i(\bsigma) = \sum_{(a_1, \ldots, a_n) \in A} \left(\pi_i(a_1, \ldots, a_n) \prod_{k=1}^n \sigma(a_k) \right)=\]
	\[=\sum_{(a_{1}, \ldots, a_{n}) \in A}\left(\pi_j(a_{\tau(1)}, \ldots, a_{\tau(n)})\prod_{k=1}^n \sigma(a_{\tau(k)})\right) = E\pi_j(\bsigma).\]
	
To show the existence of a mixed strategies superrational profile, it will suffice to prove the existence of a maximum for one of the players. 
	
	Since each $A_i$ is a finite set, $\Delta A_i$ is a compact subset of $\RR^k$, where $k$ is the number of elements in $A_i$. The set $\Gamma^{\Delta}_{G}$ is a closed subset of the compact $\Delta A_1 \times \ldots \times \Delta A_n$, and therefore compact. 	Since $E\pi_i$ restricted to $\Gamma^{\Delta}_{G}$ is a continuous function from a compact set to $\RR$, it reaches its maximum value, so any of the points where it does is a superrational profile of mixed strategies.
\end{proof}

\begin{example}
	In Example \ref{Bad} the mixed superrational profile is $(\sigma^*,\sigma^*)$, where $\sigma^*$ is the probability distribution that assigns probability $\frac{1}{2}$ to playing each of the actions. The pure strategy superrational profiles are not mixed strategy superrational profiles. Notice that  $(\sigma^*,\sigma^*)$ is also the only non degenerate mixed Nash equilibrium of the game.
\end{example}

It could be speculated that resorting to superrationality in mixed strategies may be useful for analyzing non-symmetric games. Of course this is not possible if $A_i \neq A_j$ for any pair $i,j \in I$, since then a distribution $\sigma_i \in \Delta A_i$ cannot be compared with a $\sigma_j \in \Delta A_j$ because their domains are different.

The following example shows that even if the strategy sets of the players are all the same, superrationality cannot be applied if the underlying game $G$ is not symmetric:

\begin{example} The Battle of the Sexes of Example \ref{BoS} is a non-symmetric game with two pure Nash equilibria (\textit{Box,Box}) and (\textit{Ballet, Ballet}) as well as a mixed strategy Nash equilibrium $(\sigma^{*}_1,\sigma^{*}_2)$ where   $\sigma^{*}_{1}(\mbox{\textit{Box}}) = \frac{2}{3}$, and  $\sigma^{*}_{2}(\mbox{\textit{Box}})= \frac{1}{3}$.

There is no superrational solution in pure or mixed strategies since  both  $\mathcal{SR}_{G}$ and $\mathcal{SR}^\Delta_{G}$ are empty. To see the latter case, consider profiles $(p, p) \in \Gamma^{\Delta}_{G}$, in which $p$ is the probability of \textit{Box}. A superrational solution involves maximizing both	 $E\pi_1(p,p)$ and $E\pi_2(p,p)$. We have that $E\pi_1(p,p)= 2 p^2 + 1 (1 -2p +p^2)$$= 1 - 2p + 3 p^2$, which yields a maximum $p_1^{*} = \frac{1}{3}$, while $E\pi_2(p,p)= 1 p^2 + 2 (1 -2p +p^2)$$= 2 - 4p + 3 p^2$
which yields $p_2^{*} = \frac{2}{3}$. Since $p_1^{*} \neq p_2^{*}$, no superrationality obtains.
\end{example}

\section{Superrational types}\label{harsanyi}

The definition of superrationality is not explicit about one of its fundamental components, the willingness of players to choose a superrational profile of actions. That is, superrationality assumes that the players ``have reasons'' for choosing a superrational profile. One of those reasons involves the conception the players have of the decision problem they face. Since only a limited view of the problem can justify in general choosing superrationally justifiable actions, we want to include in our models the beliefs that can lead a player to have such an viewpoint. Thus, the analysis falls straight in the realm of {\em epistemic game theory} \cite{brandenburgerbook}. One of the main approaches in this area was introduced by John C. Harsany, who presented the notion of {\em type} to derive a game with complete information from one with incomplete information \cite{harsanyi67games1}. In informal terms this concept amounts to summarize all the information and beliefs possessed by a player. Once the types of the players are determined, the game becomes one of complete information in the sense that there is no external information relevant to the decisions they make. Since beliefs are customarily represented by probability distributions and conditionals on them, Harsanyi named the players of such a game {\em `Bayesian'}. Each player's beliefs have to take into account what the other players' beliefs may be, so one gets an infinite regression of beliefs about beliefs about beliefs... As a way out of this hurdle, Harsanyi proposed that ``certain attributes'' of the players, or the players themselves are ``drawn at random from certain hypothetical populations containing a mixture of individuals of different ``types'', characterized by different attribute vectors''\cite{harsanyi67games1}.

While in Harsanyi's formulation the focus is on a few quantifiable attributes, later research, starting with \cite{boge79solutions} focused on building a space of types in which all possible beliefs are accounted for.  A quite general frame for such constructions is the theory of coalgebras\footnote{The theory of coalgebras uses the language of category theory.  A {\em coalgebra} for a functor $F$ in a category $\C$ is an object $X$ of the category together with a morphism from $X$ to $F(X)$. This quite simple definition covers a wide number of systems in which an observable behavior is a key feature. For a general introduction to this aspect of the theory of coalgebras, see \cite{rutten00universal}.} on measurable spaces \cite{moss04harsanyi}.   For a game with $n$ players we consider coalgebras in the category $\Meas^n$  where $\Meas$ is the category of measurable spaces and measurable functions between them. Let $\Delta$ be now the functor that sends a measurable space $X$ to the space of all probability distributions definable on $X$. To make $\Delta X$ into a measurable space, we endow it with the $\sigma$-algebra of subsets generated by the family of sets $\set{\mu\in\Delta X:\mu(E)\ge p}$ for all $p\in[0,1]$ and $E$ measurable in $X$. For a measurable function $f:X\to Y$,  $\mu\in\Delta X$ and $E$ a measurable subset of $Y$,  $(\Delta f)(\mu)(E)=\mu(f^{-1}(E))$. Using this notation, if we denote with $\rho_i$  the projection of a product measurable space $X_1\times\ldots\times X_n$ into $X_i$, then the marginal of a measure $\mu\in \Delta(X_1\times\ldots\times X_n)$ over $X_i$ is $Marg_{X_i}\mu=(\Delta\rho_i)\mu=\mu\circ\rho_i^{-1}$. Given a point $x$ in a measurable space $X$, the {\em Dirac measure} $\delta_x$ is the measure that gives $1$ when applied to any measurable set containing $x$ and $0$ to the rest.

Let $T_i$ be a measurable space for each $i\in I$ and let  $T_{-i}$ be the product $\prod_{j\neq i}T_j$ for $j, i\in\set{1,2,\ldots n}$. If  $K_i$ is a set of things about which player $i$ is uncertain, then we can use the functor $F:\Meas^n\to\Meas^n$ defined by
\[F(T_1, \ldots, T_n)=(\Delta(K_1\times T_{-1}),\ldots, \Delta(K_n\times T_{-n}))\]
so that a coalgebra for $F$ is an $n$-tuple of measurable spaces $(T_1,\ldots,T_n)$ and an $n$-tuple of  functions $(f_1,\ldots,f_n)$ such that $f_i: T_i\to\Delta(K_i\times T_{-i})$. The idea behind this setting is each space $T_i$ is the space of {\em types} for player $i$. For each type $t$ in $T_i$, $f_i(t)$ is a probability distribution describing the beliefs held by player $i$ if she is of type $t$, about their uncertainties in $K_i$ and the types of the other players, represented by $T_{-i}$.

\[\xymatrix{
	(T_1\ar[d]_{f_1}\mbox{,}& \ldots&,T_n\ar[d]_{f_n})\\
	( \Delta(K_1\times T_{-1}) ,&\ldots&,\Delta(K_n\times T_{-n})) }
\]

We consider a symmetric game in which the uncertainties of the players are the actions of the other players, thus we have $K_i=A_{-i}$. By Proposition \ref{existSRprofile}, there exists at least one superrational profile and therefore, at least one superrationally justifiable action. Notice that now all the uncertainty sets are isomorphic. In order to have the condition of believing the other players are of the same type we need the pool of available types to be the same for all players. We  appeal to the existence of a \emph{universal type space} $T$, that is, a type space in which all possible types can be found \cite{mertens84formulation}, \cite{moss04harsanyi}, and assume that $T_i=T$ for all $i\in I$. Universal type spaces have the property that they are isomorphic to the space into which they map, so in this case, we have that $T_i$ is isomorphic to $\Delta(A_{-i}\times T_{-i})$. Let $\rho_j$ and $\rho'_j$ be the projections from $A_{-i}\times T_{-i}$ to $A_j$ and $T_j$, respectively, so $\Delta\rho_j(f_i(t))$ represents the beliefs of player $i$ about the possible actions of player $j$, while $\Delta\rho'_j(f_i(t))$ represents the beliefs about player $j$'s type.

\[\xymatrix{
	&T_i	\ar[d]_{f_i}&\\
	\Delta A_j&  \Delta(A_{-i}\times T_{-i}) \ar[l]^{\Delta\rho_j} \ar[rr]_{\Delta\rho'_j}&&\Delta T_j
}
\]

\begin{definition}
 A \emph{superrational type} is an element $t\in T_i$ such that:
 \begin{itemize}
 	\item    for all $j\neq i$, $\Delta\rho'_j(f_i(t))=\delta_t$ 
 	\item there exists a superrationally justifiable action $a\in A_i$ such that for all $j\in I\setminus\set{i}$, $\Delta\rho_j(f_i(t))=\delta_a$.
\end{itemize}
\end{definition}
The first condition expresses, through the use of the Dirac measure,  that if player $i$ is of type $t$, then she is certain  that all the other  players are also of type $t$.  Similarly, the second condition expresses the certainty that the other players will all play the same superrationally justifiable action $a$.

The superrationally justifiable action $a\in A_i$ exists by Proposition \ref{existSRprofile}, so  a type of this kind is possible, and therefore it exists in the corresponding universal type space.

Even if a player's type is superrational, the player may choose a different action than the one associated to her type. That means that a player with a superrational type may play a non-superrationally justifiable action, as highlighted in the Prisoner's Dilemma of Example \ref{prisoner}.  To find conditions on the players that will lead to superrational profiles, we look at the strategies available to the players. In the bayesian setting the following concept  associates an action to each possible type of a player:

\begin{definition}
\cite{heifetz08epistemic}	A \emph{bayesian strategy} for player $i$ is a function $\beta_i$ from $T_i$ to $A_i$. 
\end{definition}

The conception the player has of the decision problem she faces determines the choice of the bayesian strategy. When the determination of the bayesian strategy is guided just by the goal of maximizing the payoff functions, we say the players are \textit{rational}. Yet, if we want to model superrationality, we must fashion the way the player conceives of the problem, that is, that she restricts her available options according to her beliefs.

\begin{definition}
	A \emph {superrational  bayesian strategy} for player $i$ is a function $\beta_i:T_i\to A_i$, such that for each superrational type $t$ in $T_i$, if  $\Delta\rho_j(f_i(t))=\delta_a$, then $\beta_i(t)=a$. 
\end{definition}

 \begin{theorem} \label{supertype}
 	In a  game in which  a single  superrationally justifiable action $a$ exists,  if each player has superrational type and uses a superrational bayesian strategy,  then the superrational profile $\ba=(a,\ldots,a)$ obtains. 
 \end{theorem}

\begin{proof}
		If each player has a superrational type, and is using a superrational bayesian strategy, then they all must be playing the same action $a$ which is the only superrationally justifiable one.
\end{proof}

\begin{example}
	The uniqueness condition is necessary in the previous theorem. If more than one action is superrationally justifiable,  the coordination problem of choosing which one to play arises. In the following symmetric version of the game of Example \ref{justin}, both players could have superrational types and superrational bayesian strategies and still play different actions, and therefore obtaining a suboptimal payoff.
	\begin{center}
		\begin{tabular}{c|c|c|c}
			&a&b&c\\\hline
			a& $(3,3)$&$(0,0)$&$(0,0)$\\ \hline
			b& $(0,0)$&$(3,3)$&$(0,0)$\\ \hline
			c& $(0,0)$&$(0,0)$&$(2,2)$
		\end{tabular}
	\end{center}
	
	 On the other hand, in Example \ref{prisoner}, if the second player is of superrational type, she would choose action $C$, disregarding the possibility of getting a better payoff by choosing $D$. 
\end{example}

Theorem~\ref{supertype} gives sufficient conditions for a superrational outcome in a game. While seemingly stringent, they can be compared to the epistemic conditions required to ensure a Nash equilibrium. In \cite{aumann95epistemic}, Aumann and Brandenburger give conditions, further elaborated in \cite{osborne94course}, that amount to saying that for a Nash equilibrium to be reached, all the players need to \textit{know} the actions of the others and assume that they are rational as well, meaning that their actions are the best possible responses to their beliefs. This can be expressed, in the case of pure strategy Nash equilibria, that the conditions for reaching a profile $\ba^{*}= (a^{*}_1, \ldots, a^{*}_n) \in{\mathcal{NE}_G}$ are that there exist a profile of types $(t_1, \ldots, t_n)$ and a profile of bayesian strategies $(\beta_1, \ldots, \beta_n)$ such that for every $i$, and for all $j\neq i$, $\Delta\rho_j(f_i(t_i))=\delta_{a^*_j}$ and $\beta_i(t_i)=a^{*}_i$, where $a^*_i$ is the best response to $\ba^{*}_{-i}$.    

The situation is pretty much the same if we consider mixed strategies in symmetric games. Now we have to replace the sets of pure strategies $A_i$ by their corresponding sets of mixed strategies $\Delta A_i$. Let $\gamma_j$ be the projection from $(\prod_{k\neq i}\Delta A_k)\times T_{-i}$ to $\Delta A_j$.  In this setting, by Proposition \ref{existence_mixed}, there exist at least one superrationally justifiable mixed strategy $\sigma^*$.  The existence of the universal type space establishes that there is a type $t$ such that $\Delta\gamma_j(f_i(t))=\delta_{\sigma^*}$.

\begin{definition}
	A \emph {superrational  bayesian mixed strategy} for player $i$ is a function $\alpha_i:T_i\to \Delta A_i$, such that for each superrational type $t$ in $T_i$, if  $\Delta\gamma_j(f_i(t))=\delta_\sigma$, then $\alpha_i(t)=\sigma$. 
\end{definition}
Having defined superrational bayesian mixed strategies, we can now  claim a result like Theorem \ref{supertype} for mixed strategies:

 \begin{theorem} \label{supertype}
 	In a  game in which  a single  superrationally justifiable mixed strategy $\sigma$ exists,  if each player $i$ has superrational type $t_i$ and uses a superrational bayesian mixed strategy $\alpha_i$,  then the superrational profile $\bsigma=(\sigma,\ldots,\sigma)$ obtains. 
 \end{theorem}

\section{Superrationality in strategic belief models}

The previous analysis adapted Harsanyi's framework, in which players are uncertain about the final payoffs of their actions to the case in which the uncertainty is on the actions to be carried out by the other players. In \cite{brandenburger03existence}, Brandenburger addresses this \emph {strategic uncertainty} defining the  notion of an $S$-based (interactive) possibility structure, extended jointly with Keisler  in \cite{brandenburger06impossibility} to define  {\em strategic belief models}. We adapt that definition here for the case of $n$ players and in the general framework of coalgebras.

 Let $T_1, \ldots , T_n$ be the sets of types from which each of the players are chosen,  and consider for each $i\in I$ the set of all nonempty subsets of $U_{-i}=\prod_{j\neq i} (A_j\times T_j)$. We call this set $\N(U_{-i})$.

  We consider the category $\Set^n$ which has n-tuples of sets as objects and n-tuples of functions in $\Set$, acting componentwise, as morphisms.    In this category, we use the functor that  sends an n-tuple of sets $(T_1,\ldots ,T_n)$ to the n-tuple $(\N(U_{-1}),\ldots , \N(U_{-n}))$.   A coalgebra for this functor is an n-tuple of sets  $(T_1,\ldots,T_n)$ together with an $n$-tuple of functions $(f_1,\ldots, f_n)$  with  $f_i:T_i\to  \N(U_{-i})$ for all $i=1,\ldots, n$.

A final coalgebra for this functor does not exist (unless each $A_i$ is a singleton or empty, in which case it is trivial), and this is the main result in \cite{brandenburger03existence}. We can still investigate how to characterize superrational types in this context.  The situation changes if one replaces the functor $\N$ by $\N_\omega$ that assigns to each set the set of all its finite nonempty subsets \cite{worrell99terminal}. This would amount to the  reasonable assumption that the agents can only entertain  finite sets of beliefs.

Thus we are considering coalgebras for the functor:
\[F'(T_1, \ldots, T_n)=(\N_\omega(U_{-1}),\ldots , \N_\omega(U_{-n})).\]

We call a pair $(a_i,t_i)\in  A_i\times T_i$ the {\em state of the player} $i$ and an $n$-tuple $((a_1,t_1),\ldots,(a_n,t_n))$ a {\em state of the world}. A state of the world can be regarded as a state of player $i$ together with an element of $U_{-i}$.

Following the lines of the analysis of superrationality from the previous section, we may consider symmetric games in which all the players draw their type from the same set $T=T_i$. Then for each $t\in T$, $f_i(t)$ is a nonempty finite subset of $U_{-i}$. We call these \emph{BK types}  (Brandenburger-Keisler types) to distinguish them from the models of types considered before.

\begin{definition}
	A \emph{BK superrational type} is an element $t\in T$ such that $f_i(t)$ is a singleton set $\set{((a,t),\ldots,(a,t))}$ where $a$  is superrationally justifiable. Furthermore, player $i$ is in a \textit{superrational state} if her state is the pair $(a,t)$.
\end{definition}

Thus a player having a BK superrational type regards as the only possible state for the other players one in which they all have the same type as her, while being in a superrational state reflects her conception of the options available in the decision problem she faces. It follows immediately from the definitions that:

\begin{theorem}\label{sRBK}
	In a symmetric game if each player is in a superrational state and there is a single superrationally justifiable action $a$, then a superrational profile obtains. 
\end{theorem}

We can extend this result to mixed strategies by considering the functor
\[F''(T_1, \ldots, T_n)=(\N_\omega(\prod_{j\neq 1}(\Delta A_j\times T_j)),\ldots , \N_\omega(\prod_{j\neq n}(\Delta A_j\times T_j))).\] 
Here the beliefs of each type $t_i\in T_i$ are represented by a finite set of partial states of the world in which the strategies contemplated are, instead of actions as in $F'$,  mixed strategies.

Now to each BK superrational type $t$ corresponds a singleton  $\set{((\sigma,t),\ldots,(\sigma,t))}$, where $\sigma$ is a superrationally justifiable mixed strategy. A result similar to Theorem \ref{sRBK} can also be proved. Notice that this result can be applied to justify the mixed strategy superrational outcome in Example \ref{platonia}.

\section{Players with dissimilar type spaces}

An interesting problem we have only skirted so far is that to define Hofstadter's {\em superrationality} we need means to establish if the type of two players is in some way the ``same''. In a certain sense we abused of the symmetry of the game by imposing it on the type spaces of the players. We want now to be able to talk about superrationality even if the type spaces of the players are different. One way of doing this is to use a binary relation $R$ which should hold when two types in different type spaces have the same beliefs and they choose the same action. We work now with BK type spaces where the idea may be more clearly seen, so we have functions $f_i:T_i\to \N_\omega(U_{-i})$.

\begin{definition} An {\em identification relation} between type spaces $T_1,\ldots, T_n$ is  an equivalence relation $R$ on the set $\bigcup_{i\in I}T_{i}$  such that 
for all $i, j\in I$ if $t_iRt_j$, then there exists a  permutation $\tau: I \rightarrow I$ such that $\tau(i)=j$ and for all $\bar{u} =((a_1,u_1), \ldots, (a_{i-1},u_{i-1}),(a_{i+1},u_{i+1}), \ldots,(a_{n},u_{n})) \in f_i(t_i)$, there exists $\bar{v} =((b_1,v_1), \ldots, (b_{j-1},v_{j-1}),(b_{j+1},v_{j+1}), \ldots,(b_{n},v_{n}))\in f_j(t_j)$ such that for all $k\in I\setminus \set{i}$, $a_k=b_{\tau(k)}$ and $u_kRv_{\tau(k)}$. 
\end{definition}

This is to say that two types $t_i$ and $t_j$ are related by $R$ when each of the elements in the set $f_i(t_i)\subseteq U_{-i}$ can be matched through a permutation with one element of $f_j(t_j)\subseteq U_{-j}$. This matching between elements in $U_{-i}$ and $U_{-j}$ is such that for each $k\neq i$, the action that a player of type  $t_i$ believes that player  $k$ will take is the same as the one that a player of type $t_j$ thinks player $\tau(k)$ will choose, while the corresponding types are themselves related by $R$.

In this new context we need to revise the definition of {\em superrational state} of an agent $i$:

\begin{definition}
The state of an agent $i$ is a {\em superrational state}, denoted  $(a, t_i) \in \mathcal{SR}^i$ iff
 $t_i$ is a BK superrational type such that $f_i(t_i)$ is a singleton $\set{((a,t_1), \ldots, (a,t_{i-1}),(a,t_{i+1}), \ldots,(a,t_{n}))}$, where $a$ is superrationally justifiable and there exists an identification relation $R$ such that $t_jRt_i$ for all $j\neq i$.
\end{definition}
This means that $i$ believes that all other players will choose the same superrationally justifiable action $a$ and that their types are related to her own, while she herself will play action $a$. Then even with different type spaces for different players, we have a result like the one in Theorem~\ref{sRBK}:

\begin{theorem}\label{sRdifftypes}
In a symmetric game if each player $i$'s state is in $\mathcal{SR}^i$, and there is a single superrationaly justifiable action, a superrational profile obtains. 
\end{theorem}

\begin{example}

Consider again the Prisoner's Dilemma from example \ref{prisoner}, which is symmetric. Assume that the possible types of players $1$ and $2$ are given by $T_1 = \{r, s, t \}$ and $T_2 = \{u, w\}$. On the other hand, suppose that:

\[ f_1(r) = \{(C, u)\} \ \ f_1(s) = \{(D, w)\} \ \ f_1(t) = \{(D, u)\} \]

\noindent and

\[ f_2(u) = \{(C, r)\} \ \ f_2(w) = \{(D, s)\}\]

It is easy to check that the equivalence relation   $R$ generated by $\{ (r,u), (s,w)\}$ is an identification  relation. Then $ (C,r)$ and $(D,s)$ are in $\mathcal{SR}^1$, and $ (C,u)$ and  $(D,w)$ are states  in $\mathcal{SR}^2$. But only in the states $(C,r)$ and $(C,u)$ the action taken is superrationally justifiable. Then according to Theorem \ref{sRdifftypes}, the superrational profile $(C,C)$ obtains.

\end{example}

\section{Conclusions}

In this paper we assessed the reach and the limits of Hofstadter's concept of superrationality. While the notion is quite intuitive, its formalization requires to distinguish between superrationally justifiable actions and profiles, the latter constituted by the former. But this distinction poses the questions of when do superrational profiles exist and what ensures that such a profile will consist of the \emph{same} superrationally justifiable actions chosen by all the players.

We have shown that  the existence of superrational profiles for symmetric games, both for profiles of pure or mixed strategies. This makes superrationality a notion with a limited range of applicability, unlike Nash equilibrium. On the other hand, both solution concepts suffer of undeterminacy in the case that more than one superrationally justifiable action or more than equilibrium exist, respectively. 

Another question that haunted the concept of superrationality is why players would reach it in the presence of incentives to deviate, making this solution much less relevant than Nash equilibrium. The answer we found is twofold.  First we discussed the types of the players, reflecting their beliefs about the other players, and then their conception of the decision problem. Both aspects are necessary to explain why a player may choose superrational actions. 
 Unlike rationality, the notion of superrationality not only amounts to maximize a payoff, but adds the constraint of assuming that all the players will identify with each other and therefore play the same action.

There is certain evidence that in an experimental context, players behave more superrationally than rationally \cite{rubinstein2007instinctive}. This indicates that the notion of  superrationality is worthy of consideration.

A feature of this analysis is that it works only for \emph{symmetric} games. While this requirement of symmetry can in the end be lifted for type spaces, it is still binding for action sets. Further work involves to extend the notion of superrationality to games with more general notions of symmetry as those explored in \cite{tohme17symmetry}.

\bibliographystyle{plain}


\end{document}